\documentclass{article}

\usepackage{microtype}
\usepackage{graphicx}
\usepackage{subfigure}
\usepackage{booktabs} %

\usepackage{hyperref}

\usepackage[accepted]{icml2024}

\usepackage{amsmath}
\usepackage{amssymb}
\usepackage{mathtools}
\usepackage{amsthm}

\usepackage[capitalize,noabbrev]{cleveref}

\theoremstyle{plain}
\newtheorem{theorem}{Theorem}[section]

\newtheorem{lemma}[theorem]{Lemma}

\theoremstyle{definition}

\theoremstyle{remark}

\newcommand{\etal}{\textit{et al}.}
\newcommand{\ie}{\textit{i}.\textit{e}.}
\newcommand{\eg}{\textit{e}.\textit{g}.}

\usepackage[textsize=tiny]{todonotes}
\usepackage{multirow}
\usepackage{marvosym}

\icmltitlerunning{SLAB: Efficient Transformers with Simplified Linear Attention and Progressive Re-parameterized Batch Normalization}

\begin{document}

\twocolumn[
\icmltitle{SLAB: Efficient Transformers with Simplified Linear Attention and Progressive Re-parameterized Batch Normalization}

\icmlsetsymbol{equal}{*}

\begin{icmlauthorlist}
\icmlauthor{Jialong Guo}{huawei,equal}
\icmlauthor{~~~~Xinghao Chen}{huawei,equal}%
\icmlauthor{~~~~Yehui Tang}{huawei}
\icmlauthor{~~~~Yunhe Wang}{huawei}%
\end{icmlauthorlist}

\icmlaffiliation{huawei}{Huawei Noah's Ark Lab}

\icmlcorrespondingauthor{Xinghao Chen}{xinghao.chen@huawei.com}
\icmlcorrespondingauthor{Yunhe Wang}{yunhe.wang@huawei.com}

\icmlkeywords{Machine Learning, ICML}

\vskip 0.3in
]

\printAffiliationsAndNotice{\icmlEqualContribution} %

\begin{abstract}
Transformers have become foundational architectures for both natural language and computer vision tasks. However, the high computational cost makes it quite challenging to deploy on resource-constraint devices. This paper investigates the computational bottleneck modules of efficient transformer, \ie, normalization layers and attention modules. LayerNorm is commonly used in transformer architectures but is not computational friendly due to statistic calculation during inference. However, replacing LayerNorm with more efficient BatchNorm in transformer often leads to inferior performance and collapse in training. To address this problem, we propose a novel method named PRepBN to progressively replace LayerNorm with re-parameterized BatchNorm in training. 
Moreover, we propose a simplified linear attention (SLA) module that is simple yet effective to achieve strong performance. Extensive experiments on image classification as well as object detection demonstrate the effectiveness of our proposed method. For example, our SLAB-Swin obtains $83.6\%$ top-1 accuracy on ImageNet-1K with $16.2$ms latency, which is $2.4$ms less than that of Flatten-Swin with $0.1\%$ higher accuracy. 
We also evaluated our method for language modeling task and obtain comparable performance and lower latency.
Codes are publicly available at 
\href{https://github.com/xinghaochen/SLAB}{https://github.com/xinghaochen/SLAB} and \href{https://github.com/mindspore-lab/models/tree/master/research/huawei-noah/SLAB}{https://github.com/mindspore-lab/models/}.
\end{abstract}

\section{Introduction}

Introduced initially for tasks in natural language processing~\cite{vaswani2017attention}, transformer architecture has rapidly emerged as a preeminent model in the landscape of language models. Its influence has significantly expanded with the introduction of Vision Transformer (ViT)~\cite{dosovitskiy2020image}, illustrating the efficacy and versatility of transformer-based architectures. These architectures have demonstrated their capability to achieve competitive performance benchmarks in comparison to convolutional neural networks (CNNs) across diverse vision tasks~\cite{han2022survey,wang2022multimodal,zheng2023less,tang2023category,carion2020end,xu2023fdvit}. 
Due to its powerful performance, transformer has become the mainstream architecture in deep learning. 
However, the computational demands of transformer architecture pose a significant challenge, which is predominantly due to the quadratic computational complexity of its attention mechanism and the necessity for online statistic computation of LayerNorm component.

Numerous efforts have been directed towards enhancing the efficiency of transformer architecture~\cite{tang2024survey,wu2023ppt,tang2023dynamic}.
Several approaches have sought to mitigate computational complexity by limiting the scope of token interactions within self-attention mechanisms, such as downsampling the key and value matrices~\cite{wang2021pyramid}, implementing sparse global attention patterns~\cite{child2019generating}, and computing self-attention within smaller windows~\cite{tu2022maxvit, liu2021swin, dong2022cswin}. 
Meanwhile, linear attention emerges as an alternative strategy to enhance computational efficiency by breaking down the attention mechanism into linear computational cost~\cite{katharopoulos2020transformers,cai2022efficientvit,han2023flatten,you2023castling}, yet it is still a challenging task to obtain a good balance between efficiency and accuracy. 
Moreover, there are some explorations into substituting LayerNorm (LN) with BatchNorm (BN) within transformers, motivated by the additional computational overhead LayerNorm incurs during inference. 
Yang~\etal~\yrcite{yang2022unified} propose to add a BatchNorm layer in-between the two linear
layers in the feed forward network to stabilize the training.
However, there still exists a performance gap between the LayerNorm-based and BatchNorm-based transformers.

\begin{figure}[tb]
	\centering
	\includegraphics[width=0.97\columnwidth]{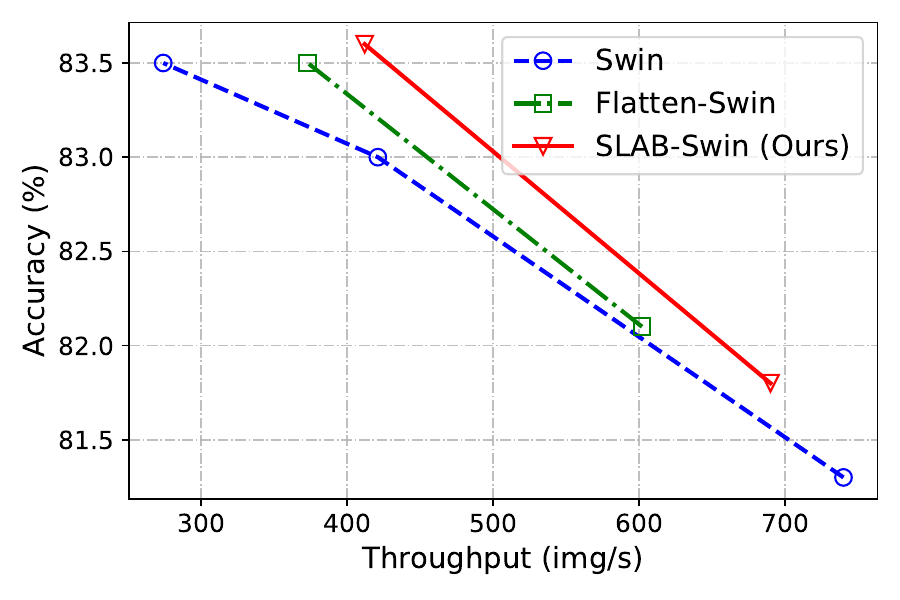}
	\vspace{-6mm}
	\caption{Comparisons of different methods on ImageNet.}
	\label{result1}
	\vspace{-8mm}
\end{figure}

In this paper, we focus on obtaining efficient transformer architectures by digging deep into the computational inefficient modules, \ie, normalization layers and attention modules.
We first explore to replace LayerNorm with BatchNorm to accelerate inference for transformer. 
BatchNorm leads to lower inference latency but may cause training collapse and inferior performance, while LayerNorm could stabilize the training yet has extra computational cost during inference. 
To this end, we first propose a progressive strategy to gradually replace LayerNorm with BatchNorm by using a hyper-parameter to control the proportion of both normalization layers. Initially the transformer architecture is dominated by the LayerNorm and gradually transits to pure BatchNorm at the end of training. This strategy effectively mitigates the risk of training collapse and also eliminating the need for calculating statistics during inference.
In addition to the progressive strategy, we also propose a novel re-parameterization formula for BatchNorm (RepBN), to enhance training stability and overall performance.

Furthermore, the computational cost of attention is critical for efficient transformer and prior methods struggle to obtain good balance of efficiency and accuracy. To this end,  we propose a simplified linear attention (SLA) module which utilizes ReLU as the kernel function and incorporate a depth-wise convolution to perform local feature enhancement. The proposed attention mechanism is more efficient than prior linear attention but still attains comparable performance.

We extensively evaluate our proposed method for various architectures on various benchmarks. Our progressive re-parameterized BatchNorm shows strong performance for image classification and object detection tasks, obtaining similar accuracy with lower inference latency. Moreover, coupled with the progressive RepBN and simplified linear attention module, our SLAB transformer achieves competitive accuracy compared to Flatten transformer with improved computational efficiency. For example, SLAB-Swin-S achieves 83.6\% Top-1 accuracy on ImageNet-1K with $16.2$ms latency, which is $2.4$ms less than that of Flatten-Swin-S with $0.1\%$ higher accuracy. We also evaluated our method for language modeling task and obtain comparable performance and lower inference latency.

\section{Related Work}

\subsection{Efficient Architecture for Transformers}
With the advent of the pioneering Vision Transformer (ViT)~\cite{dosovitskiy2020image},  the potential of the transformer architecture for computer vision tasks has been greatly explored. Various researchers are devoted to this field to make transformer-based architecture more efficient and powerful. Touvron ~\etal~\yrcite{touvron2021training} propose DeiT which utilizes distillation to achieve strong performance with training only on ImageNet1K. Liu~\etal~\cite{liu2021swin} propose Swin Transformer, which introduces shifted windowing scheme and brings greater efficiency. As the self-attention computation is limited to a small window, this transformer has linear computational complexity. Several works improve the design of sparse pattern to enhance the interaction of each token, such as CSwin~\cite{dong2022cswin}. Besides, the dynamic attention mechanism tries to control the key/value interact with query adaptive to data, such as DAT++~\cite{xia2023dat++} and BiFormer~\cite{zhu2023biformer}.

Apart from the above methods, linear attention is a popular research direction to reduce the computational complexity for transformer. Many effective mechanisms have been proposed to replace the softmax function. For example, Performers~\cite{choromanski2020rethinking} uses positive orthogonal random features approach to approximate softmax. Hydra attention~\cite{bolya2022hydra} selects the cosine similarity as kernel function. Flatten Transformer~\cite{han2023flatten} designs a focused function to improve the focus ability of linear attention.

\begin{figure*}[h]
	\begin{center}
		\includegraphics[width=0.99\linewidth]{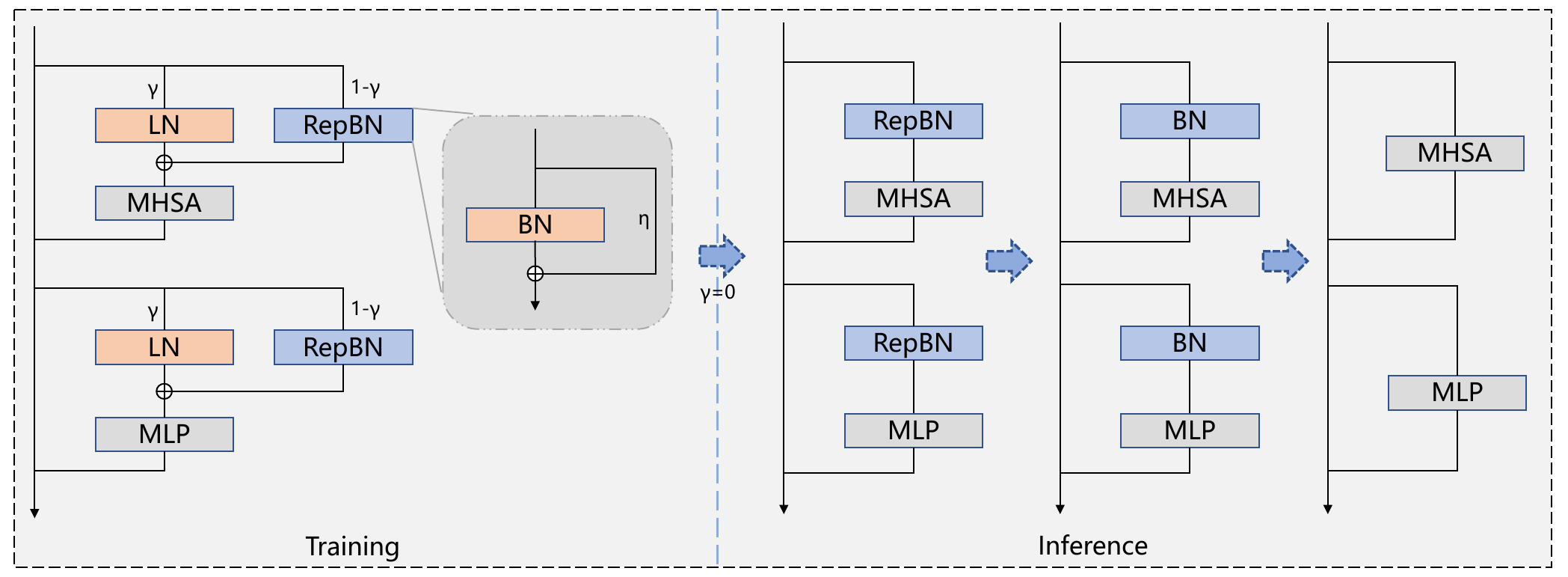}
		\vspace{-3mm}
		\caption{The overall framework of our proposed Progressive Re-parameterized BatchNorm. (a) During training, we progressively replace LayerNorm with RepBN, which is a new re-parameterization formula of BatchNorm to further improve the performance. (b) We could get $\gamma=0$ during inference, thus the transformer block transits to a RepBN-based architecture, which could further be re-parameterized to BatchNorm and merged with linear layers.}
		\label{model-structure}
	\end{center}
	\vskip -0.2in
\end{figure*}

\subsection{Normalization for Transformers}
Normalization is known as an useful method to make training stable and boost performance. Nowadays, a variety of normalization methods have been proposed, such as BatchNorm~\cite{ioffe2015batch}, LayerNorm~\cite{ba2016layer}, InstanceNorm~\cite{ulyanov2016instance}, GroupNorm~\cite{wu2018group}, MABN~\cite{yan2020towards} and UN~\cite{yang2022unified}. BatchNorm is widely used in convolutional networks and LayerNorm is commonly utilized for networks such as transformer and LSTM.

Normalization could be categorized into offline methods and online methods according to whether the mean and variance need to be computed at inference time~\cite{yang2022unified}. 
Online methods are usually  batch-irrelevant like LayerNorm, InstanceNorm and GroupNorm. These methods compute the statistics in both training and inference. LayerNorm is a commonly used in transformer architecture.

Offline methods are batch-related like BatchNorm and UN, in which the batch dimension is concluded in the calculations of both mean and variance~\cite{yao2021leveraging}. As the mean and variance are pre-computed in inference, offline normalization can be fused into adjacent linear operations. During inference, there will be no offline normalization operations and the inference time will be reduced. However, offline methods usually face the problem of performance degradation and training collapse while using in transformer. To address this problem, Yao ~\etal~\yrcite{yao2021leveraging} proposes to add a BatchNorm layer in-between the two linear layers in the MLP block that makes training statistics stable. Yang ~\etal~\yrcite{yang2022unified} finds that the issue is caused by abnormal behaviors of activation statistics, and proposes a tailored fluctuation smoothing strategy and an adaptive outlier filtration strategy to boost performance and stable training.

\section{Preliminaries}

Given the input $N$ tokens $X \in \mathbb{R}^{N \times C}$, where $C$ is the feature dimension, the general architecture of transformer block can be written as:
\begin{equation}
	\begin{split}
	X = X + \mathrm{Attn}(\mathrm{Norm}(X)), \\
	X = X + \mathrm{MLP}(\mathrm{Norm}(X)),
	\end{split}
\end{equation}
where $\mathrm{Attn}(\cdot)$ calculates the attention scores, $\mathrm{MLP}(\cdot)$ denotes multilayer perceptron and $\mathrm{Norm}(\cdot)$ is the normalization function.
In the default configuration of transformer block, $\mathrm{Norm}(\cdot)$ is usually a LayerNorm operation and $\mathrm{Attn}(\cdot)$ is the softmax-based attention mechanism~\cite{vaswani2017attention}. 

Attention plays an important role in Transformer. Denote query, key and value matrix as $Q, K, V \in \mathbb{R}^{N \times C}$, softmax attention computes the pairwise similarity between queries and keys firstly, and leads to the quadratic computation complexity $O(N^{2}C)$ in relation to the number of queries and keys N. This makes transformer computationally expensive especially in dealing with tasks that have a long sequence input. Linear attention aims to decouple the softmax function with proper approximation or instead it with other kernel function to compute $K^{T}V$ first. With this change in computation order, the computation complexity becomes $O(NC^{2})$, which is linearly related to the number of queries and keys N.

However, LayerNorm occupies unnegligible portion of latency since it requires statistic calculation during inference. Therefore, in this paper we explore to leverage BatchNorm for building efficient transformers, which only exists in training and could be merged with preceding or sequential linear layers. Moreover, the attention module plays the most important part for transformers and the softmax-based attention mechanism is computational inefficient due to its quadratic computation complexity. In this paper, we propose a simple yet efficient form of attention, which greatly reduce the latency but also remains strong performance on various vision tasks.

\section{Methods}

\begin{figure*}[ht]
	\begin{center}
		\subfigure[DeiT]{\includegraphics[width=0.33\linewidth]{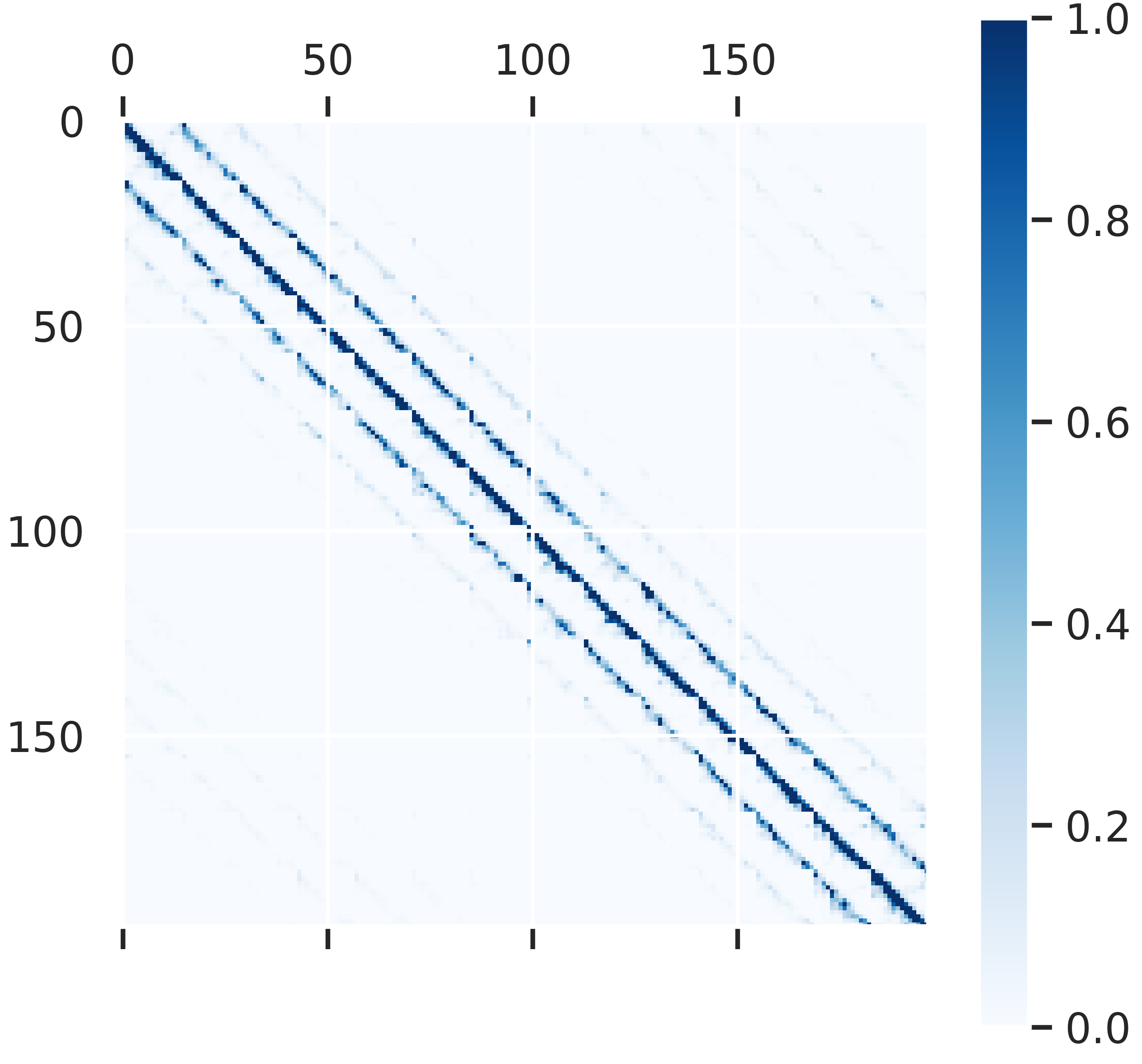}}
		\subfigure[Flatten Transformer]{\includegraphics[width=0.33\linewidth]{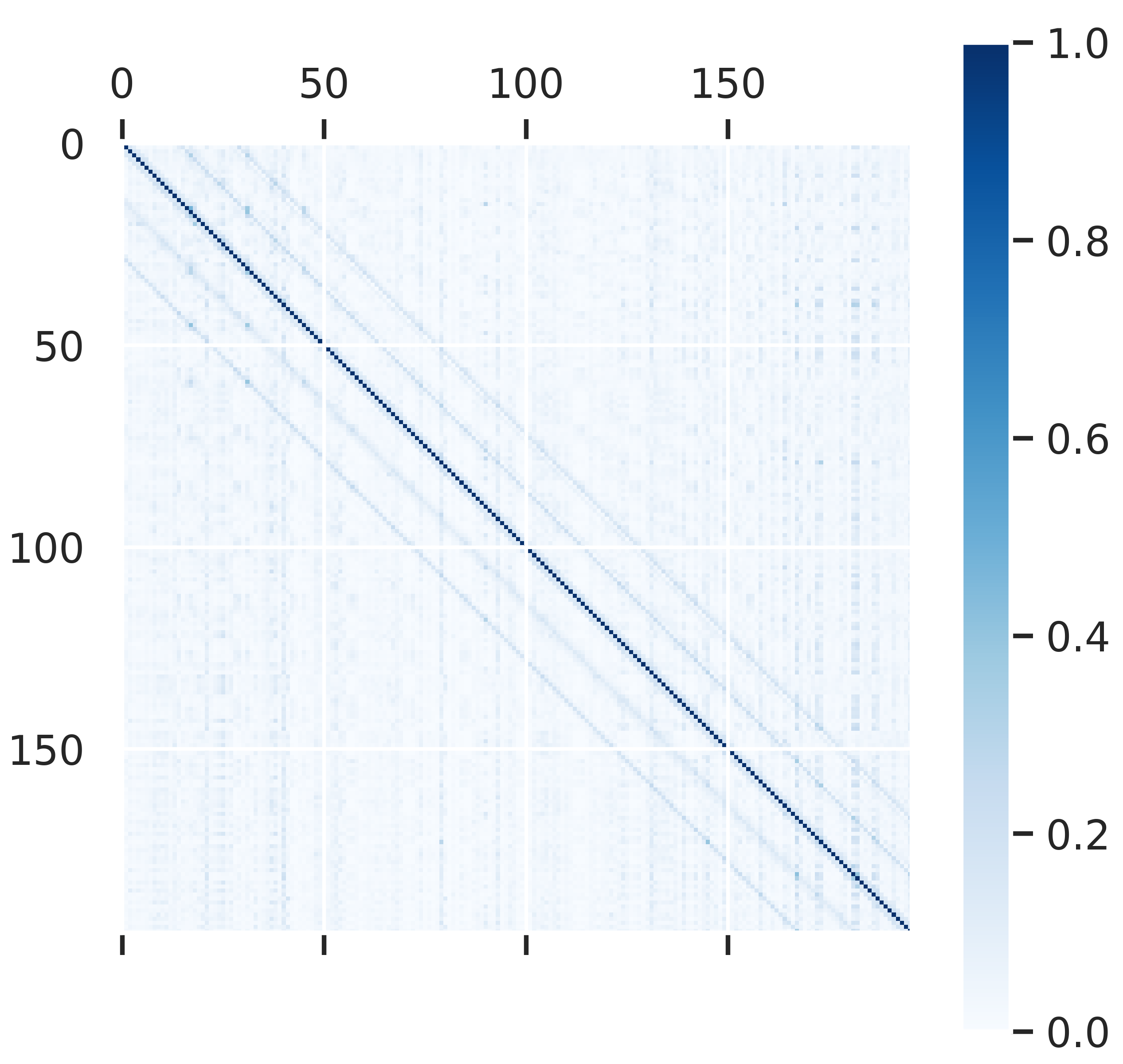}}
		\subfigure[SLAB~(Ours)]{\includegraphics[width=0.33\linewidth]{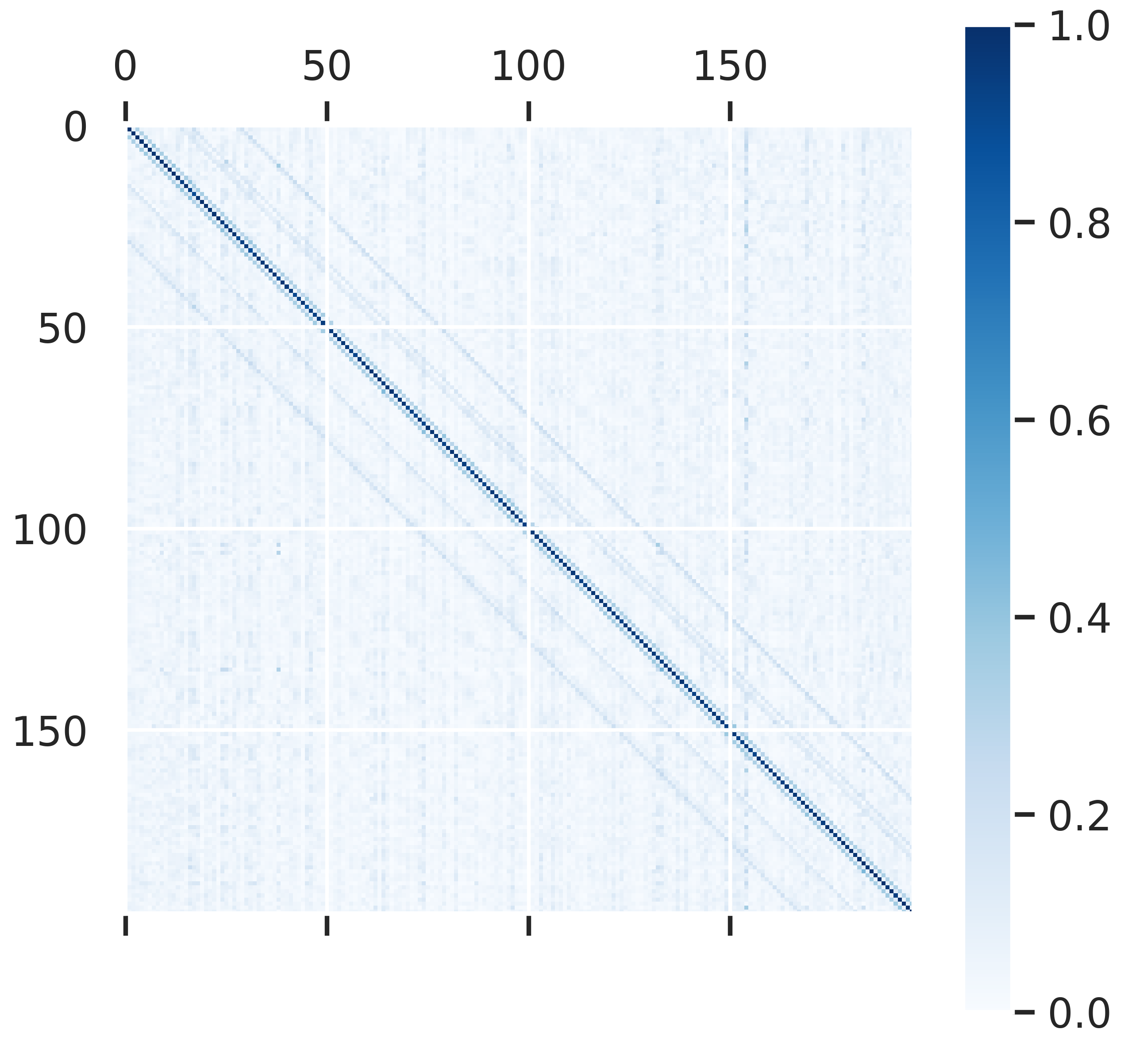}}
		\vspace{-4mm}
		\caption{Attention map ($196\times196$) from the 4rd block of the model based on DeiT-T. (a) Attention map of DeiT-T is full-rank. (b) With the help of depth-wise convolution, linear attention in Flatten Transformer has a high rank. (c) As simplified linear attention and progressive re-parameterized BatchNorm are applied in transformer, the model still keeps a high rank.}
		\label{attn-map}
	\end{center}
	\vspace{-5mm}
\end{figure*}

In this paper, we focus on building efficient transformers and propose a series of strategies, including a progressive strategy to replace the LayerNorm (LN) with the re-parameterized BatchNorm (BN) and the simplified linear attention (SLA) module. The proposed SLAB transformers obtains strong performance compared with prior methods while enjoying more computational efficacy.

\subsection{Progressive Re-parameterized BatchNorm}
LayerNorm requires statistic calculations in both training and inference, thus significantly hinders the running speed of transformers. On contrast, BatchNorm could be simply merged with linear layers during inference and is more suitable for efficient architectures. However, directly leveraging BatchNorm for transformers brings unsatisfactory performance~\cite{yao2021leveraging}. To this end, we propose to progressively replace LayerNorm with BatchNorm during training, and also propose a new re-parameterization formula of BatchNorm inspired by Ding~\etal~\yrcite{ding2021repvgg,ding2022re} to further improve the performance, as shown in Figure~\ref{model-structure}. 

\textbf{Re-parameterized BatchNorm.}
The proposed RepBN is formulated as:
\begin{equation}\label{eq:repbn}
	\mathrm{RepBN}(X) = \mathrm{BN}(X) + \eta X,
\end{equation}
where $\eta$ is a learnable parameter that is jointly trained in an end-to-end manner. Once the training is done, the RepBN could be re-parameterized as a norm form of BN, as shown in Lemma~\ref{lemma_repbn}.

\begin{lemma}\label{lemma_repbn}
	Denote a BN layer with mean $\mu$, standard deviation $\sigma$, rescale and shift parameters $\alpha$ and $\beta$ as $\mathrm{BN}(X; \mu, \sigma, \alpha, \beta)$. We can re-parameterize the RepBN in Eq.~\ref{eq:repbn} as:
	\begin{equation}\label{eq:repbn2bn}\small
		\mathrm{RepBN}(X; \mu, \sigma, \alpha, \beta) = \mathrm{BN}(X; \mu, \sigma, \alpha+\eta\sigma, \beta+\eta\mu).
	\end{equation}
\end{lemma}

\begin{proof}
	\begin{equation}\label{eq:repbn2bn_proof}
		\begin{split}
		&\mathrm{RepBN}(X; \mu, \sigma, \alpha, \beta) = \mathrm{BN}(X; \mu, \sigma, \alpha, \beta) + \eta X\\
		& =\frac{X-\mu}{\sigma}\alpha + \beta + \eta X=\frac{X-\mu}{\sigma}\alpha + \beta + \frac{X}{\sigma}\sigma\eta\\
		& =\frac{X-\mu}{\sigma}\alpha + \beta + \frac{X-\mu}{\sigma}\sigma\eta+\mu\eta\\
		& =\frac{X-\mu}{\sigma}(\alpha+\eta\sigma) + (\beta+\eta\mu)\\
		& =\mathrm{BN}(X; \mu, \sigma, \alpha+\eta\sigma, \beta+\eta\mu).
		\end{split}
	\end{equation}
\end{proof}

Based on Lemma~\ref{lemma_repbn}, the distribution of RepBN's output is control by $\alpha+\eta\sigma$ and $\beta+\eta\mu$, which is corresponds to the variance and mean. RepBN can recover the distribution with the help of $\sigma$ and $\mu$. 

Meanwhile, when $\alpha=0, \beta=0$, it is equivalent to BatchNorm being skipped. When $\eta=0$, RepBN is converted into pure BatchNorm.

\textbf{Progressive LN $\rightarrow$ RepBN.}
To facilitate the training of a pure BN-based transformers, we propose to progressively transit the LN to RepBN during training, \textit{i.e.},
\begin{equation}\label{eq:prepbn}
	\mathrm{PRepBN}(X) = \gamma\mathrm{LN}(X) + (1 - \gamma)\mathrm{RepBN}(X),
\end{equation}
where $\gamma$ is a hyper-parameter to control the output of different normalization layers. Generally $\gamma=1$ at the begin of training when the LN dominates the architecture, and $\gamma=0$ at the end of training to make sure it transits to a pure BN-based transformer. 
We utilize a simple yet effective decay strategy for $\gamma$:
\begin{equation}\label{eq:gamma}
	\gamma = \dfrac{T - T_{cur}}{T}, \gamma \in [0, 1],
\end{equation}
where $T$ is the total steps of training with LayerNorm and $T_{cur}$ is the current step. This progressive strategy eases the difficulty of training a pure BN-based transformer and thus leads to strong performance on various tasks.

There are some other decay strategies for attenuating the value of $\gamma$ gradually, such as cosine decay and step decay. Empirically, we find that the linear strategy is one of the more effective and simpler.

\subsection{Simplified Linear Attention}

Attention module is the most import part in a transformer network, which is generally formulated as:
\begin{equation}\label{eq:attention}\small
	\begin{split}
		&Q=XW_{Q}, K=XW_{K}, V=XW_{V},\\
		&O_{i} = \sum_{j=1}^{N}\dfrac{\mathrm{Sim}(Q_{i}, K_{j})}{\sum_{j}\mathrm{Sim}(Q_{i}, K_{j})}V_{j},
	\end{split}
\end{equation}
where $W_Q, W_K, W_V \in \mathbb{R}^{C \times C}$ project the input tokens to query, key and value tensors, respectively. $\mathrm{Sim}(\cdot, \cdot)$ denotes the similarity function.
For the original form of attention, the similarity function is:
\begin{equation}\label{eq:softmax_attn}\small
	\mathrm{Sim_{softmax}}(Q_i , K_j) = \exp (\frac{Q_iK_j^{T}}{\sqrt{C}}),
\end{equation}
this softmax-based attention leads to high computational complexity. Several recent methods investigate the usage of linear attention to remove the softmax calculation thus improve the efficacy of transformers~\cite{han2023flatten}. 
However, these methods still suffer quite complex design and are not computation efficient enough.
In this paper, we propose a simplified linear attention (SLA) which is formulated as follow:
\begin{equation}\label{eq:sla}
	\small
	\begin{split}
		&{\rm Sim}_{SLA}\left(Q_{i},K_{j}\right)=\mathrm{ReLU}\left(Q_{i}\right){\mathrm{ReLU}\left(K_{j}\right)}^T,\\
		&\tilde {\rm O}_{i} = \sum_{j=1}^{N}\dfrac{\mathrm{Sim}_{SLA}(Q_{i}, K_{j})}{\sum_{j}\mathrm{Sim}_{SLA}(Q_{i}, K_{j})}V_{j},\\
		&\!{\rm O}_{SLA}\!=\tilde {\rm O}+\!{\rm DWC}(V),
	\end{split}
\end{equation}
where $DWC(\cdot)$ denotes a depth-wise convolution.
It is a simple yet efficient linear attention since it also enjoys the decoupling computation order by computing $K^TV$ first, and leads to great complexity reduction.
Moreover, only ReLU function and depth-wise convolution are explored and both operations are computation friendly in most hardware.

To demonstrate that our method still maintains feature diversity, we visualize the effect of attention map on DeiT-T that applied the strategy of progressive re-parameterized BatchNorm and simplified linear attention (SLAB), as shown in Figure~\ref{attn-map}. It can be find that a high rank is still kept for our proposed method, demonstrating its good capacity for capturing attention information.

\begin{table*}[ht]
	\caption{Comparison of different normalizations for various transformer architectures on ImageNet1K.}
	\begin{center}
	\resizebox{0.95\textwidth}{!}{
	\label{tab:sota_norm}
	\small
		\begin{tabular}{l|lcc|c}
			\toprule
			\textbf{Method} & \textbf{Normalization} & \textbf{FLOPs (G)} & \textbf{Throughput} & \textbf{Top-1 Acc. (\%)} \\
			\midrule
			\multirow{ 2}{*}{DeiT-T~\cite{touvron2021training}} & LN (Default) & 1.3 & 3432 & 72.2 \\
			& \bf PRepBN (Ours) & 1.3 & \bf 4194 & \bf 73.6 \\
			\midrule
			\multirow{3}{*}{DeiT-S~\cite{touvron2021training}} & LN (Default) & 4.6 & 952 & 79.8 \\
			& BN+FFNBN~\cite{yao2021leveraging} & 4.6 & 990 & 78.8\\
			& \bf PRepBN (Ours)  & 4.6 & \bf 990 & \bf 80.2 \\
			\midrule
			\multirow{ 2}{*}{PVT-T~\cite{wang2021pyramid}} & LN (Default)  & 1.9 & 1500 & 75.1 \\
			& \bf PRepBN (Ours)   & 1.9 & \bf 1756 & \bf76.0 \\
			\midrule
			\multirow{ 2}{*}{PVT-S~\cite{wang2021pyramid}} & LN (Default) & 3.8 & 814 & 79.8 \\
			& \bf PRepBN (Ours)   & 3.8 & \bf 911 & \bf80.1 \\
			\midrule
			\multirow{ 2}{*}{PVT-M~\cite{wang2021pyramid}} & LN (Default) & 6.7 & 520 & 81.2 \\
			& \bf PRepBN (Ours)   & 6.7 & \bf 556 & \bf81.7 \\
			\midrule
			\multirow{3}{*}{Swin-T~\cite{liu2021swin}} & LN (Default)  & 4.5 & 740 & 81.3 \\
			& BN+FFNBN~\cite{yao2021leveraging} & 4.5 & 805 & 80.9\\
			& \bf PRepBN (Ours)  & 4.5 & \bf 805 & \bf 81.4 \\
			\midrule
			\multirow{3}{*}{Swin-S~\cite{liu2021swin}} & LN (Default) & 8.7 & 421 & 83.0 \\
			& BN+FFNBN~\cite{yao2021leveraging} & 8.7 & 452 & 82.8\\
			& \bf PRepBN (Ours)  & 8.7 & \bf 452 & \bf 83.2 \\
			\midrule
			\multirow{3}{*}{Swin-B~\cite{liu2021swin}} & LN (Default) & 15.4 & 274 & 83.5 \\
			& BN+FFNBN~\cite{yao2021leveraging} & 15.4 & 284 & 83.1\\
			& \bf PRepBN (Ours)  & 15.4 & \bf 284 & \bf 83.6 \\
			\bottomrule
		\end{tabular}
	}
	\end{center}
	\vskip -0.17in
\end{table*}

\begin{figure*}[ht]
	\vskip -0.05in
	\begin{center}
		\subfigure[DeiT]{\includegraphics[width=0.33\linewidth]{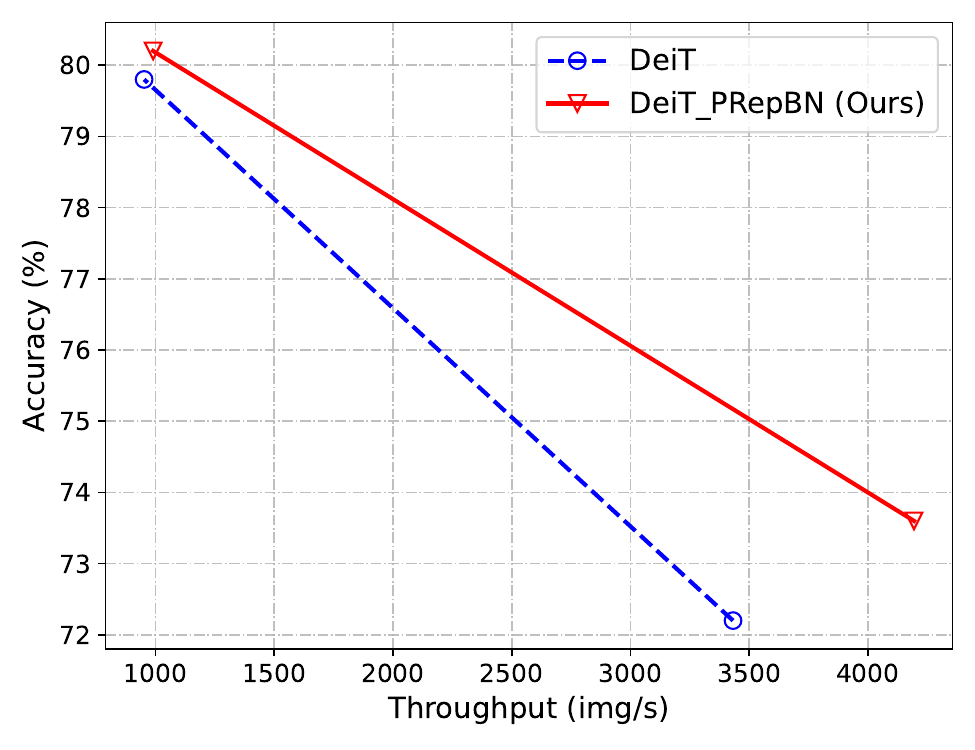}}
		\subfigure[PVT]{\includegraphics[width=0.33\linewidth]{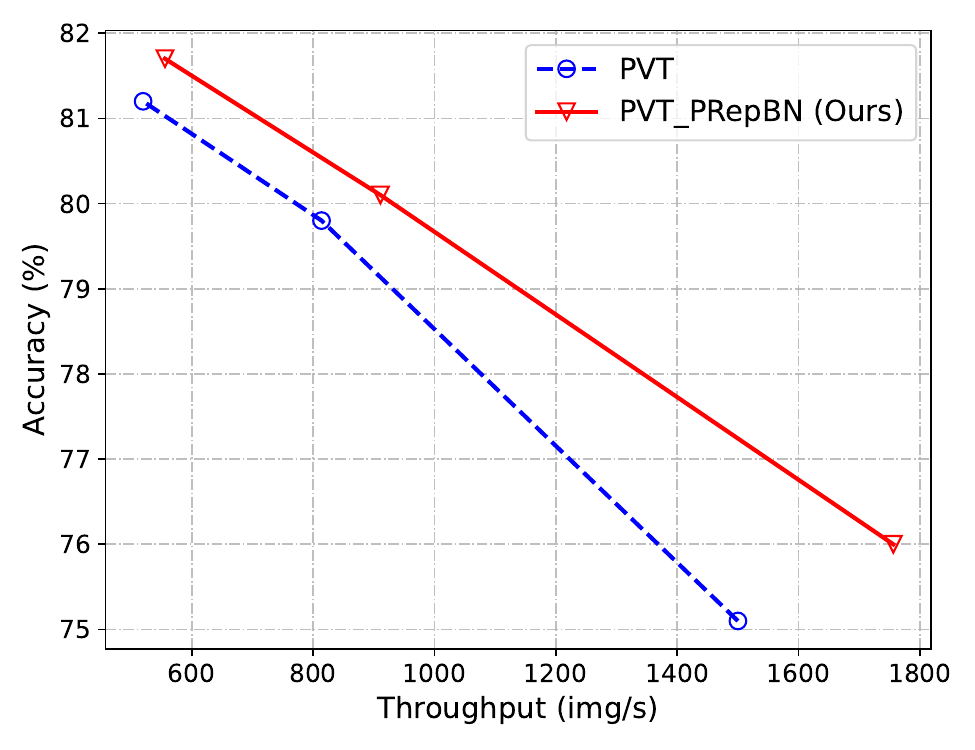}}
		\subfigure[Swin]{\includegraphics[width=0.33\linewidth]{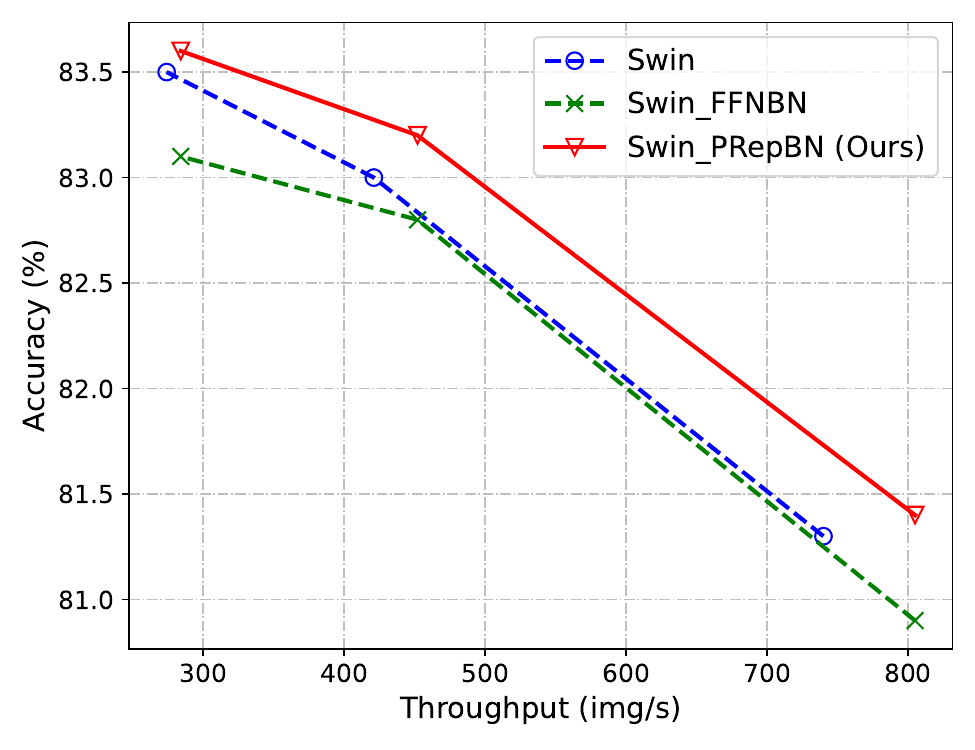}}
		\vspace{-1.5em}
		\caption{Comparisons of accuracy and throughput for different methods on ImageNet1k.}
		\label{acc-thr}
	\end{center}
	\vskip -0.25in
\end{figure*}

\section{Experiments}\label{sec:exp}

In this section, we evaluate our method on various computer vision tasks including image classification, object detection and instance segmentation and also on language modeling task. We conduct extensive experiments for various backbones with our proposed progressive re-parameterized BatchNorm and simplified linear attention module. We train our model on ImageNet-1K~\cite{deng2009imagenet} for image classification, and evaluate the effect of object detection and instance segmentation tasks on COCO dataset~\cite{lin2014microsoft}. At last, we ablate the important design elements of our proposed method on classification task.

\subsection{Image Classification}

\textbf{Settings.} For image classification, we adhere to the configuration outlined in ~\cite{touvron2021training}. We train all models for 300 epochs with AdamW optimizer, incorporating a cosine decay learning rate scheduler with 20 epochs of linear warm-up. The batch size is set to 1024, the initial learning rate is 0.001, and the weight decay value is 0.05. In the case of models utilizing the proposed progressive re-parameterized BatchNorm, a reduced droppath~\cite{larsson2016fractalnet} rate is applied. The linear decay steps $T$ for PRepBN slightly varies across different backbones. Due to the variance shift induced by droppath, we freeze the model parameters and exclusively update the statistics of re-parameterized BatchNorm for 10 epochs at the end of training. 
We also demonstrate the effectiveness  of our method with both progressive re-parameterized BatchNorm and simplified linear attention. We follow the setting of ~\cite{han2023flatten} on macro architecture design and training. %
All reported results of throughput/latency are obtained on a single V100 GPU. For classification task, we measure FLOPs as well as the throughput/latency for the image resolution of 224$\times$224.

\textbf{Results on image classification task.} Table~\ref{tab:sota_norm} presents the results of different backbones with our PRepBN normalization. 
Our proposed PRepBN demonstrates comparable or even superior performance when compared with LayerNorm. More specifically, the models using our PRepBN as normalization layer exhibit performance improvements ranging from 0.1\% to 1.4\%. Notably, PRepBN is amenable to fusion with other linear operations, allowing it to obtain more efficient inference.
We further compare our method with BN+FFNBN~\cite{yao2021leveraging} which also aims to train transformer model with BatchNorm. It can be seen that our PRepBN achieves consistent improvements on different backbones. For example, our proposed PRepBN achieves 80.2\% top-1 accuracy on DeiT-S model, which is 1.4\% better than the BN+FFNBN method. For swin transformer, our PRepBN brings +0.5\%, +0.4\% and +0.5\% accuracy gain than BN+FFNBN on Swin-T, Swin-S and Swin-B models.
As shown in Figure~\ref{acc-thr}, we present a comparative analysis of our method across DeiT-based, PVT-based, and Swin-based models. It is evident that transformers equipped with our PRepBN achieve higher throughput while maintaining similar accuracy levels.

\begin{table*}[ht]
	\caption{Comparison of different linear transformers on ImageNet1K.}
	\begin{center}
	\resizebox{0.94\textwidth}{!}{
	\setlength{\tabcolsep}{16pt}
	\small
	\label{tab:sota}
		\begin{tabular}{l|cc|c}
			\toprule
			\textbf{Method}  & \textbf{FLOPs (G)} & \textbf{Latency (ms)} & \textbf{Top-1 Acc. (\%)} \\
			\midrule
			Flatten-DeiT-T~\cite{han2023flatten}  & 1.1 & 15.2 & 74.1\% \\
			\bf SLAB-DeiT-T (Ours) & 1.1 & \bf 9.6 & \bf 74.3\% \\
			Flatten-DeiT-S~\cite{han2023flatten}  & 4.4 & 15.5 & 80.4\% \\
			\bf SLAB-DeiT-S (Ours) & 4.4 & \bf 10.4 & \bf 80.0\% \\
			\midrule
			Flatten-PVT-T~\cite{han2023flatten}  & 2.0 & 10.8 & 77.8\% \\
			\bf SLAB-PVT-T (Ours) & 2.0 & \bf 8.0 & \bf 76.5\% \\
			\midrule
			Flatten-CSwin-T~\cite{han2023flatten}  & 4.3 & 32.4  & 83.1\% \\
			\bf SLAB-CSwin-T (Ours) & 4.3  & \bf 29.3  & \bf 82.8\% \\
			\midrule
			Flatten-Swin-T~\cite{han2023flatten}  & 4.5 & 10.9 & 82.1\% \\
			\bf SLAB-Swin-T (Ours)  & 4.5 & \bf 8.7 & \bf 81.8\% \\
			Flatten-Swin-S~\cite{han2023flatten}  & 8.8 & 18.6 & 83.5\% \\
			\bf SLAB-Swin-S (Ours)  & 8.7 & \bf 16.2 & \bf 83.6\% \\
			\bottomrule
		\end{tabular}
	}
	\end{center}
	\vskip -0.2in
\end{table*}

\begin{table*}[ht]
	\caption{Mask R-CNN Object Detection \& Instance Segmentation on COCO.}
	\resizebox{\textwidth}{!}{
	\setlength{\tabcolsep}{0.9pt}
	\small
	\label{tab:sota_coco}
		\begin{tabular}{l|c|c|ccc|ccc|ccc|ccc}
			\toprule
			\textbf{Method}  & \textbf{Schd.} & \textbf{Lat. (ms)} & $\mathbf {AP^{b}}$ & \textbf{$AP^{b}_{50}$} &  \textbf{$AP^{b}_{75}$} & \textbf{$AP^{b}_{s}$} & \textbf{$AP^{b}_{m}$} & \textbf{$AP^{b}_{l}$} & $\mathbf {AP^{m}}$ & \textbf{$AP^{m}_{50}$} & \textbf{$AP^{m}_{75}$} & \textbf{$AP^{m}_{s}$}  & \textbf{$AP^{m}_{m}$}  & \textbf{$AP^{m}_{l}$} \\
			\midrule
			Swin-T~\cite{liu2021swin}  & $1\times$ & 51.0 &  43.1 & 66.0 & 47.0 & 27.3 & 46.4 & 56.0 &  39.4 & 62.8 & 42.0 & 23.0 & 42.9 & 53.7 \\
			\bf Swin-T-PRepBN (Ours)  & $1\times$ & \bf 43.0 & \bf 42.9 & 65.8 & 46.8 & 27.7 & 46.2 & 55.4 & \bf 39.3 & 62.6 & 41.9 & 23.1 & 42.8 & 53.6 \\
			\midrule
			Swin-S~\cite{liu2021swin}  & $1\times$ & 63.1 &  46.2 & 68.7 & 50.9 & 30.2 & 49.7 & 60.8 &  41.7 & 65.6 & 44.5 & 25.1 & 45.5 & 57.3 \\
			\bf Swin-S-PRepBN (Ours)  & $1\times$ &\bf 57.3 & \bf 45.9 & 68.4 & 50.2 & 29.8 & 49.3 & 60.2 & \bf 41.4 & 65.0 & 44.7 & 24.9 & 44.9 & 57.0 \\
			\midrule
			PVT-T~\cite{wang2021pyramid}  & $1\times$ & 46.0 & 36.6 & 58.9 & 39.2 & 21.3 & 38.9 & 48.9 &  34.9 & 56.2 & 37.0 & 19.6 & 37.2 & 48.2 \\
			\bf PVT-T-PRepBN (Ours)  & $1\times$ &\bf 43.2 &\bf 36.5 & 59.0 & 39.2 & 20.9 & 38.4 & 50.1 & \bf 34.4 & 55.7 & 36.5 & 18.7 & 36.3 & 48.9 \\
			\midrule
			PVT-S~\cite{wang2021pyramid}  & $1\times$ &64.0 &  40.5 & 63.2 & 44.1 & 23.4 & 43.4 & 54.9 &  37.9 & 60.2 & 40.6 & 20.8 & 40.6 & 52.6 \\
			\bf PVT-S-PRepBN (Ours)  & $1\times$ &\bf 59.6 & \bf 40.6 & 63.3 & 43.9 & 24.4 & 43.1 & 55.4 & \bf 38.0 & 60.6 & 40.5 & 21.4 & 40.5 & 53.5 \\
			\midrule
			Flatten-Swin-T~\cite{han2023flatten}  & $1\times$ & 60.0 &  44.2 & 67.3 & 48.5 & 29.4 & 47.5 & 57.0 &  40.2 & 63.8 & 43.0 & 24.5 & 43.8 & 54.7 \\
			\bf SLAB-Swin-T (Ours)  & $1\times$ & \bf 54.1 & \bf 43.9 & 66.5 & 48.1 & 28.6 & 47.4 & 56.7 & \bf 40.1 & 63.3 & 43.1 & 24.3 & 43.8 & 54.2 \\
			\bottomrule
		\end{tabular}
	}
	\vspace{-1em}
\end{table*}

Table~\ref{tab:sota} presents the performance of our SLAB transformer, which is powered by our proposed progressive re-parameterized BatchNorm and simplified linear attention module. We compare our model  with Flatten transformer, which utilizes focused linear attention for higher efficiency. 
Experiments  on various architectures including DeiT~\cite{touvron2021training}, PVT~\cite{wang2021pyramid}, CSwin~\cite{dong2022cswin} and Swin~\cite{liu2021swin} demonstrate than our SLAB transformer obtains better performance than Flatten transformer. More specifically, our SLAB-Swin-T model obtains $83.6\%$ top-1 accuracy on ImageNet-1K with $16.2$ms latency, which is $2.4$ms less than that of Flatten-Swin with $0.1\%$ higher accuracy. Our models are more computational efficient mainly due to more hardware friendly normalization layers as well as the simplified linear attention module.
Figure~\ref{result1} also shows the trade-off between accuracy and latency of our SLAB transformer, Flatten tranformer~\cite{han2023flatten} and the original Swin transformer, which demonstrate better performance of our model.

\subsection{Object Detection}

\textbf{Settings.} We use Mask R-CNN~\cite{he2017mask} to evaluate the effectiveness of our method on COCO  dataset for object detection and instance segmentation tasks. The backbones used in Mask R-CNN are pretrained on ImageNet-1K. 
All models are trained for $1\times$ schedule, \ie, 12 epochs. The latency is measured with a batch size of 1 on V100 GPU for a average value of 100 rounds.

\begin{figure}[tb]
	\centering
	\includegraphics[width=0.99\columnwidth]{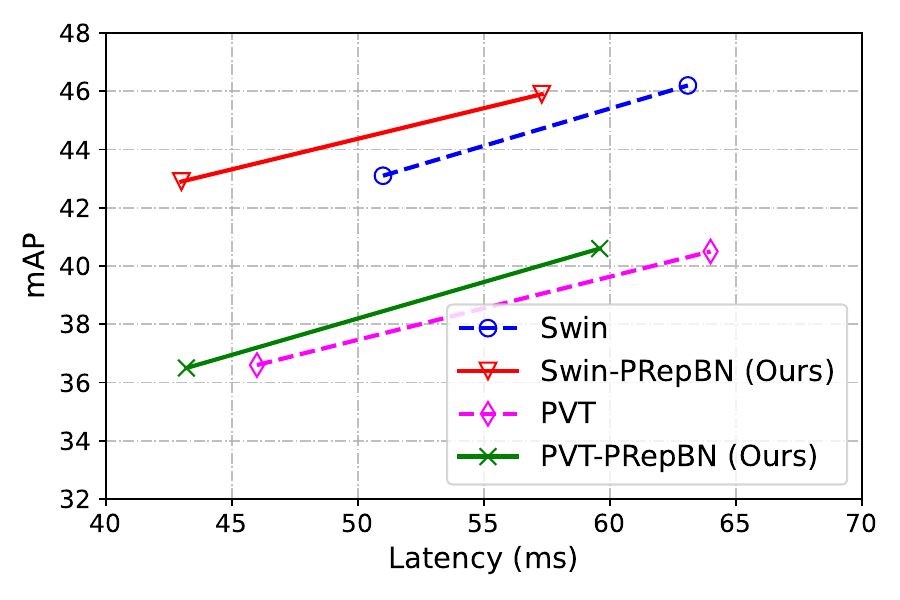}
	\vspace{-5mm}
	\caption{Comparison of different normalization on COCO.}
	\label{coco_det}
	\vspace{-7.8mm}
\end{figure}

\begin{table*}[ht]
	\caption{Results of perplexity on Wikitext-103 dataset. 256/480 indicate evaluation context window sizes.}
	\setlength{\tabcolsep}{9pt}
	\renewcommand{\arraystretch}{1.1}
	\small
	\label{tab:lm}
	\vspace{-0mm}
	\begin{center}
		\begin{tabular}{l|c|cccc|c}
			\toprule
			\multirow{ 2}{*}{\textbf{Model}} & \multirow{ 2}{*}{\textbf{\# Param.}} & \multicolumn{ 2}{|c}{\textbf{256}} & \multicolumn{ 2}{c|}{\textbf{480}} & \multirow{ 2}{*}{\textbf{Inference Time(ms/t)}} \\
			& & \textbf{Val.} & \textbf{Test} & \textbf{Val.} & \textbf{Test} & \\
			\midrule
			Adaptive Inputs w/ LN & 247M & 19.5 & 20.2 & 19.3 & 20.0 & 13.9 \\
			\bf Adaptive Inputs w/ PRepBN (Ours) & 247M & \bf 19.2 & \bf 20.0 & \bf 19.1 & \bf 19.8 & \bf 12.9 \\
			\bottomrule
		\end{tabular}
	\end{center}
	\vspace{-7mm}
\end{table*}

\begin{table*}[!ht]
	\caption{Experimental results of the proposed method for LLaMA-350M on various benchmarks.}
	\resizebox{\textwidth}{!}{
	\setlength{\tabcolsep}{1.5pt}
	\renewcommand{\arraystretch}{1.12}
	\small
	\label{tab:llama}
	\centering
	\begin{tabular}{l|c|ccccccccc|l}
		\toprule
		\bf Model & \bf Throughput & \bf ARC-C & \bf ARC-E & \bf BoolQ & \bf COPA & \bf HellaSwag & \bf PIQA & \bf WinoGrande & \bf OpenBookQA & \bf SciQ & \bf Avg. \\ 
		\midrule
		LLaMA-350M & 44.0 & 22.95 & 46.13 & 59.27 & 64 & 33.19 & 64.36 & 49.09 & 29.6 & 75.3 & 49.32 \\ 
		\bf w/ PRepBN (Ours)& \bf 50.4 & 23.55  & 44.44 & 60.67 & 66 & 32.76  & 64.04 & 49.72 &30.0 &76.8 & \bf 49.78 \\
		\bottomrule
	\end{tabular}
	}
\vspace{-5mm}
\end{table*}

\textbf{Results on object detection task.} We compare our proposed PRepBN with standard LayerNorm for various backbones including Swin and PVT for object detection task.  The results are shown in Table~\ref{tab:sota_coco}. 
It reveals that our method achieves quite comparable performance with original models equipped with LayerNorm. Taking advantages of offline normalization, the models with our proposed PRepBN obtain lower inference latency. For example, the latency of Mask R-CNN exhibits a reduction from 64ms to 59.6ms when PVT-S backbone is equipped with our PRepBN and the accuracy for object detection and instance segmentation is similar. 
A more clear visualization for the trade-off between mAP and latency is also presented in Figure~\ref{coco_det}, which demonstrates that our proposed method achieves better overall performance on object detection.

\subsection{language modeling}
\textbf{Settings.} We also evaluate our proposed method on language modeling task based on Adaptive Inputs~\cite{baevski2018adaptive}. We train our models on the Wikitext-103~\cite{merity2016pointer} dataset, which contains over 100 million tokens.  We set the number of tokens per GPU to 4096 and train on 8 GPUs. The number of tokens per sample is limit to 512. %
We also apply our method on LLaMA-350M model, following the similar architecture and training settings as prior work~\cite{yang2023gated,he2024densemamba}.

\textbf{Results on language modeling task.} As shown in Table~\ref{tab:lm}, our PRepBN achieves similar perplexity with the model equipped with LayerNorm, while the latency reduces from 13.9 ms to 12.9 ms per token.  Besides, We apply our PRepBN on more modern large language models such as LLaMA, which adopts the variant of LayerNorm that removes the computation of mean values, \ie, RMSNorm. 
As shown in Table~\ref{tab:llama}, our method successfully boost the throughput from 44.0 to 50.4 tokens per second on V100 GPU, and obtains even slightly better average accuracy. These results demonstrate the effectiveness of our proposed PRepBN on language modeling task.

\begin{table}[tb]
	\vspace{-3mm}
	\caption{Ablation studies for the impact of simplified linear attention and progressive re-parameterized BatchNorm.}
	\label{tab:abla_module}
	\resizebox{\linewidth}{!}{
	\setlength{\tabcolsep}{7.5pt}
	\vspace{-1mm}
	\begin{tabular}{l|cc|c}
		\toprule
		\textbf{Method}  & \textbf{FLOPs} & \textbf{Lat. (ms)} & \textbf{Acc. (\%)} \\
		\midrule
		Flatten-DeiT-T  & 1.1 G & 15.2 & 74.1 \\
		+ SLA  & 1.1 G & 10.2 & 73.0 \\
		+ SLA + PRepBN  & 1.1 G &\bf  9.6 & \bf 74.3 \\
		\midrule
		Flatten-PVT-T  & 2.0 G & 10.8 & 77.8 \\
		+ SLA  & 2.0 G & 8.5 & 75.2 \\
		+ SLA + PRepBN  & 2.0 G &\bf  8.0 &\bf  76.5 \\
		\midrule
		Flatten-Swin-T & 4.5 G & 10.9 & 82.1 \\
		+ SLA  & 4.5 G & 9.5 & 81.9 \\
		+ SLA + PRepBN  & 4.5 G &\bf  8.7 &\bf  81.8 \\
		\midrule
		Flatten-Swin-S & 8.8 G & 18.6 & 83.5 \\
		+ SLA  & 8.8 G & 18.0 & 83.4 \\
		+ SLA + PRepBN  & 8.7 G &\bf  16.2 &\bf  83.6 \\
		\bottomrule
	\end{tabular}
	}
	\vspace{-7mm}
\end{table}

\subsection{Ablation Studies}

In this section, we conduct extensive ablation studies to demonstrate the impact of our key designs.

\textbf{The impact of SLA and PRepBN.} We first explore the impact of the simplified linear attention (SLA) module and progressive re-parameterized BatchNorm (PRepBN) on different backbones. As shown in Table~\ref{tab:abla_module}, utilizing our simplified linear attention (SLA) brings consistent improvement for efficiency. For DeiT and PVT, our SLA obtains significant latency reduction and a few accuracy drop. Moreover, Swin transformers equipped with our SLA achieve quite comparable accuracy with that of original ones but with lower latency. 
In addition, the latency could be further reduced by replacing LayerNorm by our proposed progressive re-parameterized BatchNorm (PRepBN). 
This strategy hardly affects the accuracy and even recover the accuracy of model like DeiT and PVT. 
Combining these two strategies, the latency is reduced by 5.6 ms when the accuracy is improved by 0.2\% for DeiT-T. Moreover, our method obtains similar accuracy and harvests 2.2 ms and 2.4 ms latency reduction for Swin-T and Swin-S models. 

\begin{table}[tb]
	\vspace{-2mm}
	\caption{Ablation studies for the impact of progressive strategy and re-parameterized BatchNorm.}
	\vspace{-1mm}
	\setlength{\tabcolsep}{16pt}
	\renewcommand{\arraystretch}{1.1}
	\label{tab:abla_bn}
	\begin{center}
		\begin{tabular}{l|c}
			\toprule
			\textbf{Method} & \textbf{Acc. (\%)} \\
			\midrule
			DeiT-T-BN & 71.9 \\
			+ Progressive Strategy  & 73.1 \\
			+ Progressive Strategy + RepBN & \bf 73.6 \\
			\bottomrule
		\end{tabular}
	\end{center}
	\vspace{-9mm}
\end{table} 

\textbf{Ablation study for PRepBN.} 
We investigate key components of our proposed PRepBN, \ie, the progressive strategy and re-parameterized BatchNorm (RepBN). Directly training a BatchNorm-based transformer leads to quite unstable training, either obtaining inferior performance  or collapse in training (\eg, DeiT-S and Flatten-Swin-T).  
To avoid the variance shift~\cite{li2019understanding} caused by droppath, which will influence the performance of BatchNorm, we simply set the droppath rate to 0 on DeiT-T model. 
As shown in Table~\ref{tab:abla_bn}, applying progressive strategy on a  BatchNorm-based DeiT-T model brings 1.2\% accuracy gain. 
We further utilize our RepBN in the model and the accuracy increases to 73.6\%. These results demonstrate that both our proposed progressive strategy and re-parameterized BatchNorm (RepBN) are beneficial for training a pure BatchNorm-based transformer. 

\section{Conclusion}

In this paper, we investigates the computational bottleneck modules of transformer and propose novel strategies including progressive Re-parameterized BatchNorm  and simplified linear attention to obtain efficient transformer architectures. 
Our method progressively replace LayerNorm with re-parameterized BatchNorm  during training to obtain lossless accuracy, while leveraging the efficiency advantages of BatchNorm during inference. Additionally, we devise a simplified linear attention mechanism that attains  comparable performance with other linear attention methods but with less computational cost. Through extensive experiments for both computer vision and language modeling tasks, we showcase that our method achieves stronger performance with respect to accuracy and efficiency than prior methods and sheds light into the design of efficient transformer.

\noindent \textbf{Acknowledgements.} 
We gratefully acknowledge the support of MindSpore~\cite{mindspore}, CANN (Compute Architecture for Neural Networks) and Ascend AI Processor used for this research.

\section*{Impact Statements}
This paper presents work whose goal is to advance the field of Deep Learning. There are many potential societal consequences of our work, none which we feel must be specifically highlighted here.

\bibliography{example_paper}
\bibliographystyle{icml2024}

\newpage
\appendix
\onecolumn
\setcounter{table}{0}
\renewcommand{\thetable}{A\arabic{table}}

\section{Detailed hyper-parameter settings.}

The detailed hyper-parameter settings are provided in Table~\ref{tab:hyperparameters setting}.
For image classification, we follow the setting of DeiT~\cite{touvron2021training}. We use AdamW as our default optimizer and train all of model for 300 epochs with a cosine decay learning rate scheduler and 20 epochs of linear warm-up. The batch size is set to 1024, an initial learning rate is 0.001, and the value of weight decay is 0.05. For model used proposed progressive re-parameterized batchnorm, we employ an smaller rate of droppath~\cite{larsson2016fractalnet}. 
Owing to the variance shift caused by droppath, We freeze the model parameters and only update the statistics of re-parameterized BatchNorm for 10 epochs at the end of training. %
\begin{table*}[ht]
	\vskip -0.2in
	\caption{The settings of hyperparameters on different models for image classification task.}
	\setlength{\tabcolsep}{9pt}
	\small
	\label{tab:hyperparameters setting}
	\begin{center}
		\begin{tabular}{l|cc|ccc|ccc}
			\toprule
			\textbf{Name}  & \textbf{DeiT-T} & \textbf{DeiT-S} & \textbf{PVT-T} & \textbf{PVT-S} & \textbf{PVT-M} & \textbf{Swin-T} & \textbf{Swin-S} & \textbf{Swin-B} \\
			\midrule
			Epoch  & 300 & 300 & 300 & 300 & 300 & 300 & 300 & 300 \\
			Batch Size  & 1024 & 1024 & 1024 & 1024 & 1024 & 1024 & 1024 & 1024 \\
			Learning rate & 1e-3 & 1e-3 & 1e-3 & 1e-3 & 1e-3 & 1e-3 & 1e-3 & 1e-3 \\
			Warmup Steps  & 20 & 20 & 20 & 20 & 20 & 20 & 20 & 20 \\
			Optimizer & AdamW & AdamW & AdamW & AdamW & AdamW & AdamW & AdamW & AdamW \\
			Droppath Rate & 0.0 & 0.0 & 0.0 & 0.1 & 0.3 & 0.1 & 0.3 & 0.3 \\
			Linear Decay Steps  & 0 & 3e5 & 0 & 3e5 & 3e5 & 0 & 3e5 & 3e5 \\
			\bottomrule
		\end{tabular}
	\end{center}
	\vskip -0.2in
\end{table*}

\section{Combination with post-quantization method.}
Our method focuses on replacing LayerNorm with BatchNorm to obtain inference speed-up without performance degradation. It is a complementary strategy with common model compression methods like weight quantization, pruning or distillation and these methods can be combined to achieve a better performance. As a proof of concept, we conduct post-quantization using RepQ-ViT~\cite{li2023repq} on DeiT-Tiny model that is trained with our PRepBN strategy. As shown in the below table, applying W8A8 quantization on top of our method still achieves the accuracy of 73.6\%, while further reducing the computational cost to only 0.33 GFLOPs. This demonstrates that our method can be effectively combined with other compression methods such as quantification. 

\begin{table}[!ht]
	\vskip -0.2in
	\caption{The settings of hyperparameters on different models for image classification task.}
	\centering
	\small
	\begin{tabular}{l|ccc}
		\toprule
		\bf Method & \bf FLOPs (G) & \bf Throughput (images/s) & \bf Top-1 Acc (\%) \\ 
		\midrule
		DeiT-T & 1.3 & 3432 & 72.2 \\ 
		w/ PRepBN (Ours) & 1.3 & 4194 & \bf 73.6 \\ 
		w/ PRepBN (Ours) + RepQ-ViT (W8A8) & 0.33 & - & \bf 73.6 \\ 
		\bottomrule
	\end{tabular}
\end{table}

\end{document}